\documentclass[conference]{IEEEtran}
\IEEEoverridecommandlockouts
\usepackage{cite}
\usepackage{amsmath,amssymb,amsfonts}
\usepackage{graphicx}
\usepackage{textcomp}
\usepackage{xcolor}

\usepackage{url}
\usepackage{subcaption}

\usepackage{algorithm}
\usepackage{algpseudocode}

\usepackage{amsmath,amssymb, amsthm}

\theoremstyle{definition}
\newtheorem{teiri}{Theorem}
\newtheorem{hodai}{Lemma}

\newtheorem{teigi}{Definition}

\newcommand{\annot}[2]{\underbrace{#1}_{\text{#2}}}

\def\BibTeX{{\rm B\kern-.05em{\sc i\kern-.025em b}\kern-.08em
    T\kern-.1667em\lower.7ex\hbox{E}\kern-.125emX}}
\begin{document}

\title{CFTM: Continuous time fractional topic model
}

\author{\IEEEauthorblockN{1\textsuperscript{st} Kei Nakagawa}
\IEEEauthorblockA{\textit{Innovation Lab} \\
\textit{Nomura Asset Management Co, Ltd.}\\
Tokyo, Japan \\
kei.nak.0315@gmail.com}
\and
\IEEEauthorblockN{2\textsuperscript{nd} Kohei Hayashi}
\IEEEauthorblockA{
\textit{RIKEN}\\
Saitama, Japan\\
0000-0002-2237-0032
}
\and
\IEEEauthorblockN{3\textsuperscript{rd} Yugo Fujimoto}
\IEEEauthorblockA{\textit{Innovation Lab} \\
\textit{Nomura Asset Management Co, Ltd.}\\
Tokyo, Japan \\
yu5fujimoto@gmail.com}
}

\maketitle
\begin{abstract}
In this paper, we propose the Continuous Time Fractional Topic Model (cFTM), a new method for dynamic topic modeling. This approach incorporates fractional Brownian motion~(fBm) to effectively identify positive or negative correlations in topic and word distribution over time, revealing long-term dependency or roughness. 
Our theoretical analysis shows that the cFTM can capture these long-term dependency or roughness in both topic and word distributions, mirroring the main characteristics of fBm. 
Moreover, we prove that the parameter estimation process for the cFTM is on par with that of LDA, traditional topic models.
To demonstrate the cFTM's property, we conduct empirical study using economic news articles. 
The results from these tests support the model's ability to identify and track long-term dependency or roughness in topics over time.
\end{abstract}

\section{Introduction}
A topic model is a type of probabilistic generative model specifically designed for text analysis. 
In the framework of a topic model, the creation of a document is viewed as a two-step generative process. Initially, a topic is chosen based on a predefined topic distribution. 
Subsequently, words in the document are selected in accordance with the word distribution associated with the chosen topic. 
Here, the \textit{topic distribution} refers to the relative presence of each topic within a specific document, while the \textit{word distribution} denotes the frequency and arrangement of words constituting each topic.
The core of topic modeling involves statistically inferring both the topic and word distributions from a set of observed documents. This inference is predicated on the assumption that these documents have been generated according to the rules of a topic model. Through this process, it becomes possible to discern not only the proportion of various topics within the documents but also the specific word distributions that characterize each topic.
Among the various approaches to topic modeling, Latent Dirichlet Allocation~(LDA), as introduced by \cite{blei2003latent}, stands out due to its widespread adoption and effectiveness\cite{jelodar2019latent}.
Additionally, there have been significant advancements in this field, including models that account for correlations among topics, as discussed by \cite{blei2007correlated}. 
These developments reflect the ongoing evolution and increasing sophistication of topic modeling methodologies in document analysis.

A notable limitation of Latent Dirichlet Allocation~(LDA) is its static nature; it does not account for the time evolution of topic and word distribution. 
This shortcoming makes it challenging to track and analyze the changes in topics over time, which can be a critical aspect in many applications, especially those involving time-series documents. 
To address this issue, the Dynamic Topic Model~(DTM) was proposed by \cite{blei2006dynamic}, representing a significant advancement in the field of topic modeling\cite{abdelrazek2023topic}.
The DTM extends the traditional topic model framework by incorporating temporal dynamics. 
In this approach, the dataset is partitioned into segments corresponding to equal time intervals. Crucially, both the topic distribution and the word distribution within each topic are allowed to evolve over time. 
This temporal modeling capability enables DTM to capture and reflect the dynamic nature of topics, allowing for a more accurate analysis of documents that change over time.

Building upon the discrete time framework of DTM, the Continuous Time Dynamic Topic Model (cDTM) was proposed by \cite{wang2008continuous} as a significant extension. 
While DTM effectively models changes in topics at distinct, separate time points, cDTM advances this concept by allowing for modeling in a continuous time framework. 
In the DTM and cDTM, the parameters governing word and topic generation at each time point evolve over time. This evolution is modeled by adjusting the parameters at each subsequent time step through the addition of random variables, which independently follow a normal distribution. This method implies that the increments in the generative parameters are independent at each point in time.

However, empirical observations suggest that the evolution of topics and their associated word distributions often exhibit positive correlations over time. For instance, the emergence of new words or the prominence of words specific to a new topic at a given time point can influence the topic and word distributions in subsequent observations. This observation is supported by various studies \cite{hong2010empirical,wang2007mining,wang2006topics,leskovec2009meme}, which indicate that assuming a positive correlation in the changes of generative parameters can lead to a more realistic and accurate model of document generation.
Conversely, there are scenarios where topics that are prevalent at one time point may become obsolete or less discussed in the next, or where entirely new topics emerge abruptly. In such cases, it's critical to capture the negative correlation or 'roughness' in the topic distribution.

Therefore, In this study, we propose the Continuous-Time Fractional Topic Model~(cFTM), a novel topic model that encapsulates the long-term dependencies or roughness in the generative parameters of topic and word distribution. 
The cFTM distinguishes itself by employing fractional Brownian motion (fBm\cite{mandelbrot1968fractional}) as its underlying process. Fractional Brownian motion is a generalization of standard Brownian motion, characterized by its capacity to incorporate long-term memory or roughness, making it an ideal choice for modeling the incremental process of generative parameters in topic models.
The cFTM can be viewed as an extension and generalization of the cDTM. 
This relationship highlights cFTM's ability to model sequential time-series data with an arbitrary level of granularity, similar to cDTM, but with the added sophistication of handling long-term dependencies and roughness in topic evolution.
We theoretically prove that  the fluctuations in topic and word distributions within cFTM inherently reflect the long-term memory or roughness characteristics of fBm. 
Moreover, we establish that the parameter estimation for cFTM is on par with traditional topic models.
To empirically validate our model, we conduct numerical experiments using real-world datasets, specifically economic news articles. 
These experiments were instrumental in confirming the cFTM's capability to effectively capture and represent the long-term dependencies or roughness in topic evolution, thereby demonstrating its practical applicability and potential benefits in analyzing time-series documents.

\section{Related Work}
Since the introduction of Latent Semantic Indexing~(LSI) by \cite{deerwester1990indexing} and Probablistic LSI~(PLSI) by \cite{hofmann1999probabilistic}, a significant number of researchers have been actively developing various techniques in the field of topic modeling. 
Each of these models differs in terms of modeling capabilities and makes distinct assumptions about the corpus, document representation, and the nature of the topics themselves. 
Over the years, the developed topic models have covered a broad spectrum, catering to diverse use cases. 
The continuous development of new models is an ongoing effort, aiming to fill current research gaps and broaden the range of applications for topic modeling.

The advent of Latent Dirichlet Allocation (LDA\cite{blei2003latent}) marked a pivotal moment in topic modeling, leading to the development of numerous Bayesian Probabilistic Topic Models. 
These models, recognized for their efficiency, became predominant in research circles\cite{abdelrazek2023topic}. 
They are known for their ease of deployment, straightforward interpretation, and modular design that facilitates extension into more complex models. 
Furthermore, their computational efficiency, attributable to a relatively lower number of parameters~\cite{blei2003latent}, has made them a popular choice in various fields.

Topic models can aid in comprehending large document sets, proving beneficial across various disciplines like economics, finance, and sociology. For example, \cite{jeong2019social} demonstrated that topic models can identify latent product topics discussed by customers on social media, thereby quantifying the importance of each product topic. This application exemplifies how topic models can be used to understand public opinions about companies on social media.

\cite{chowdhury2019towards} employed LDA and Non-negative Matrix Factorization~(NMF) models to extract important and emerging concepts related to transportation infrastructure planning. Leveraging these concepts, they developed a preliminary taxonomy of transportation infrastructure planning, showcasing the application of topic models in exploring trends within the transportation sector.

In economics, topic models have been utilized to analyze the historical evolution and changing structures of the field. 
\cite{ambrosino2018topic} constructed a map of the discipline over time, illustrating the technique's potential in studying the shifting structure of economics, particularly during times of potential fragmentation.

In the finance sector, \cite{takano2023text} proposed methods using topic models to extract relevant information about dividend policies from annual securities reports. This approach enhances the accuracy of future dividend forecasts. Similarly, \cite{nakagawa2020good} employed LDA to summarize integrated reports, distinguishing between high and low-quality reports. Furthermore, supervised topic models, which incorporate labeled data \cite{mcauliffe2007supervised}, have been applied in brand perception analysis \cite{manabe2021identification} and stock return prediction \cite{sashida2021stock}, demonstrating their broad utility in the financial domain.

Given the abundance of topic models, there has been a growing need for their summarization and categorization. Several efforts to categorize the extensive range of topic models have been documented~\cite{alghamdi2015survey,abdelrazek2023topic}. \cite{alghamdi2015survey} introduces two primary categories: the first categorizes models into four subgroups, namely LSI, PLSI, LDA, and the Correlated Topic Model~(CTM\cite{blei2007correlated}). 
The second category focuses on the evolution of topic models over time, including models such as the DTM and the continuous-time Dynamic Topic Model~(cDTM)~\cite{wang2008continuous}. 
Our research is in line with the exploration of these evolving topic models.

\section{Preliminary}
\begin{figure*}[t]
  \begin{minipage}[b]{0.19\linewidth}
    \centering
    \includegraphics[keepaspectratio, scale=0.13]{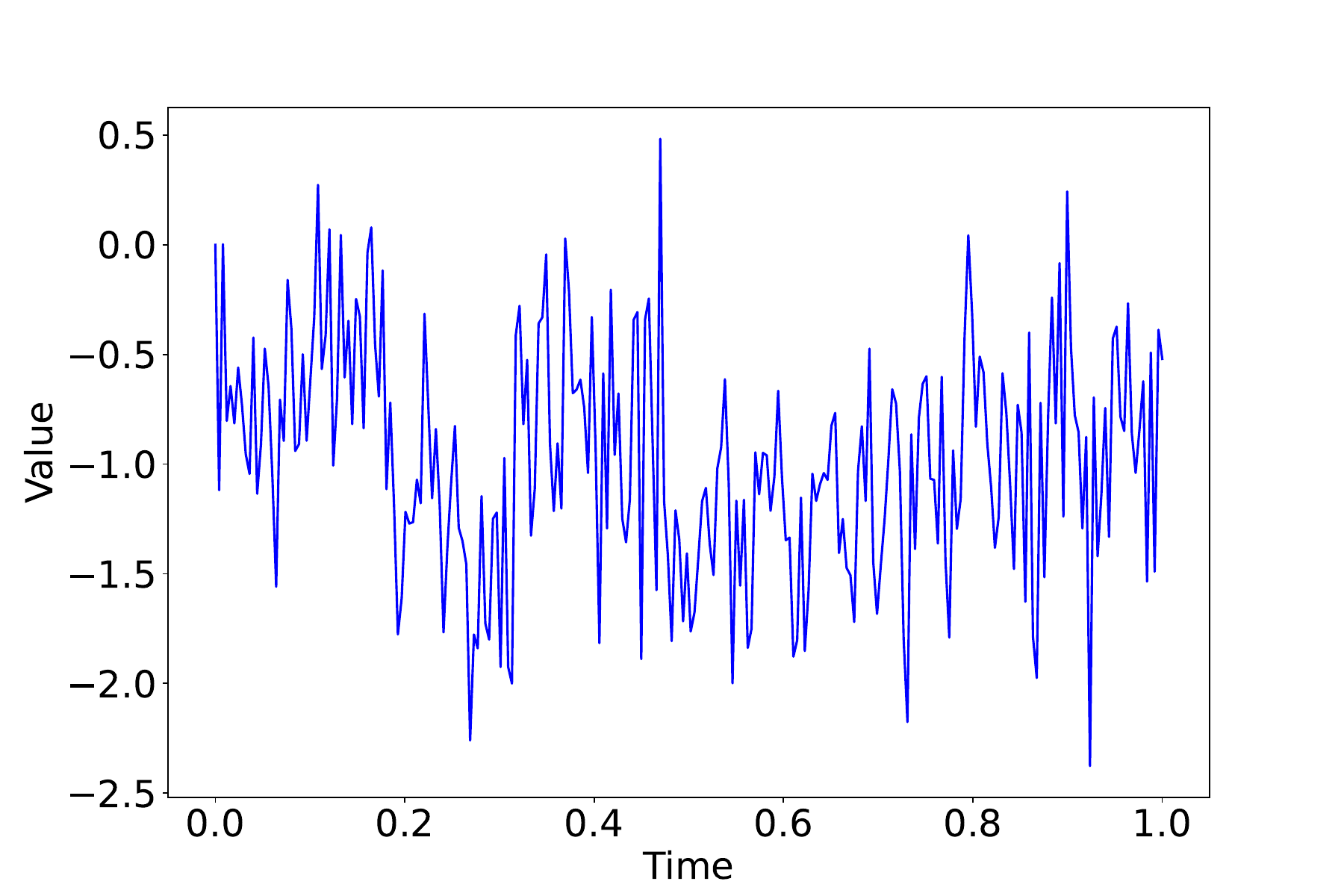}
    \caption{$H=0.1$}
  \end{minipage}
  \begin{minipage}[b]{0.19\linewidth}
    \centering
    \includegraphics[keepaspectratio, scale=0.13]{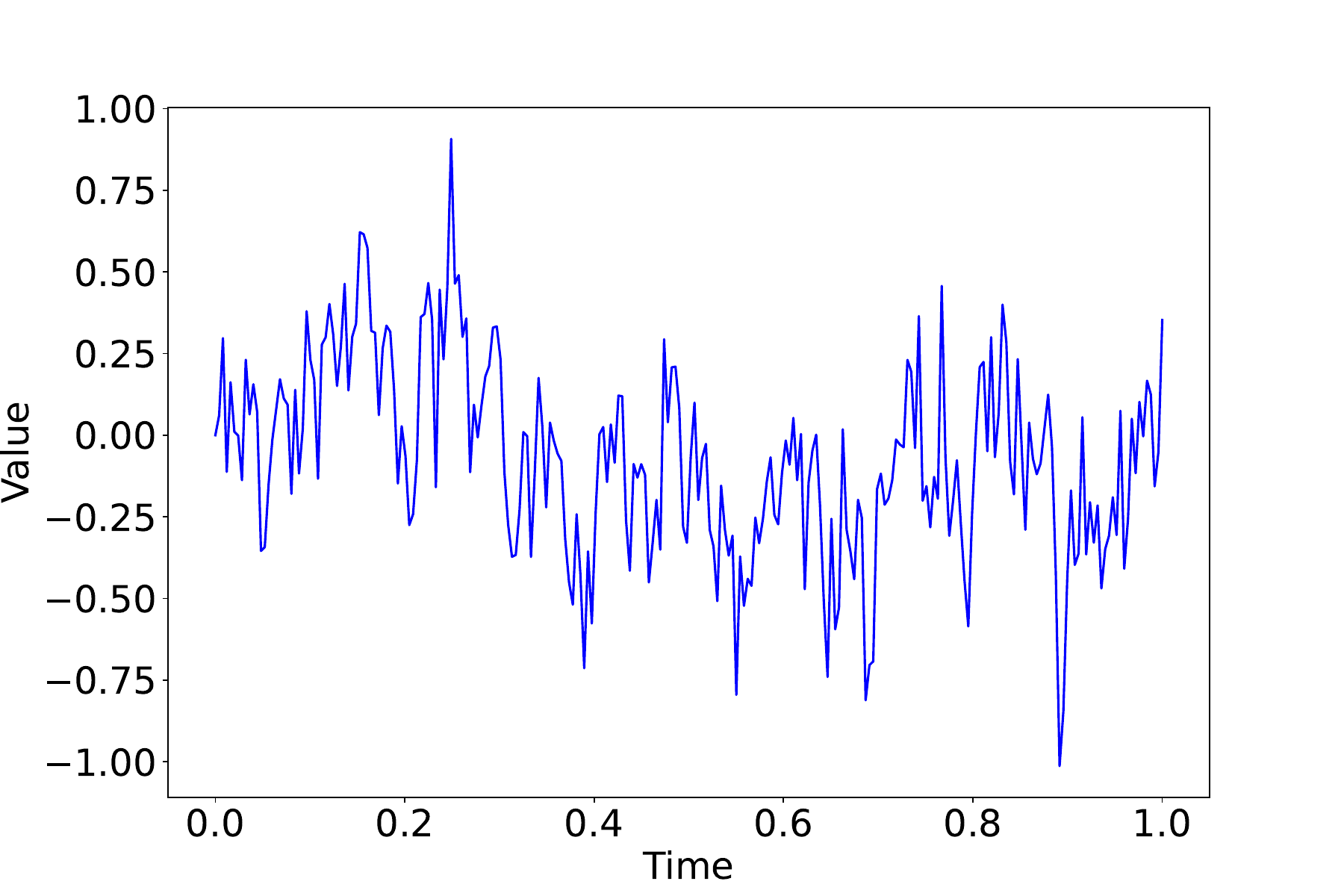}
    \caption{$H=0.25$}
  \end{minipage}
  \begin{minipage}[b]{0.19\linewidth}
    \centering
    \includegraphics[keepaspectratio, scale=0.13]{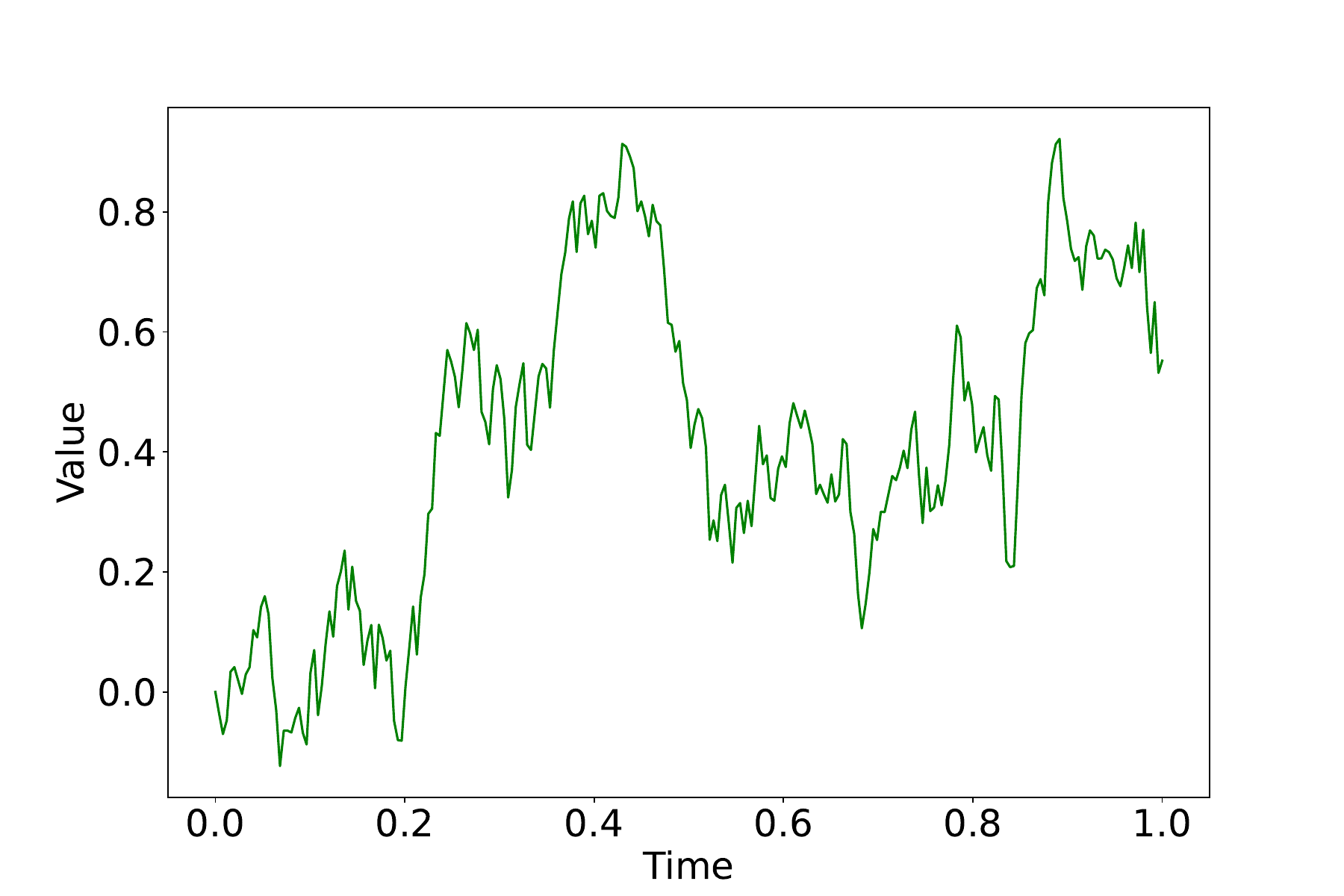}
    \caption{$H=0.5$}
  \end{minipage}
  \begin{minipage}[b]{0.19\linewidth}
    \centering
    \includegraphics[keepaspectratio, scale=0.13]{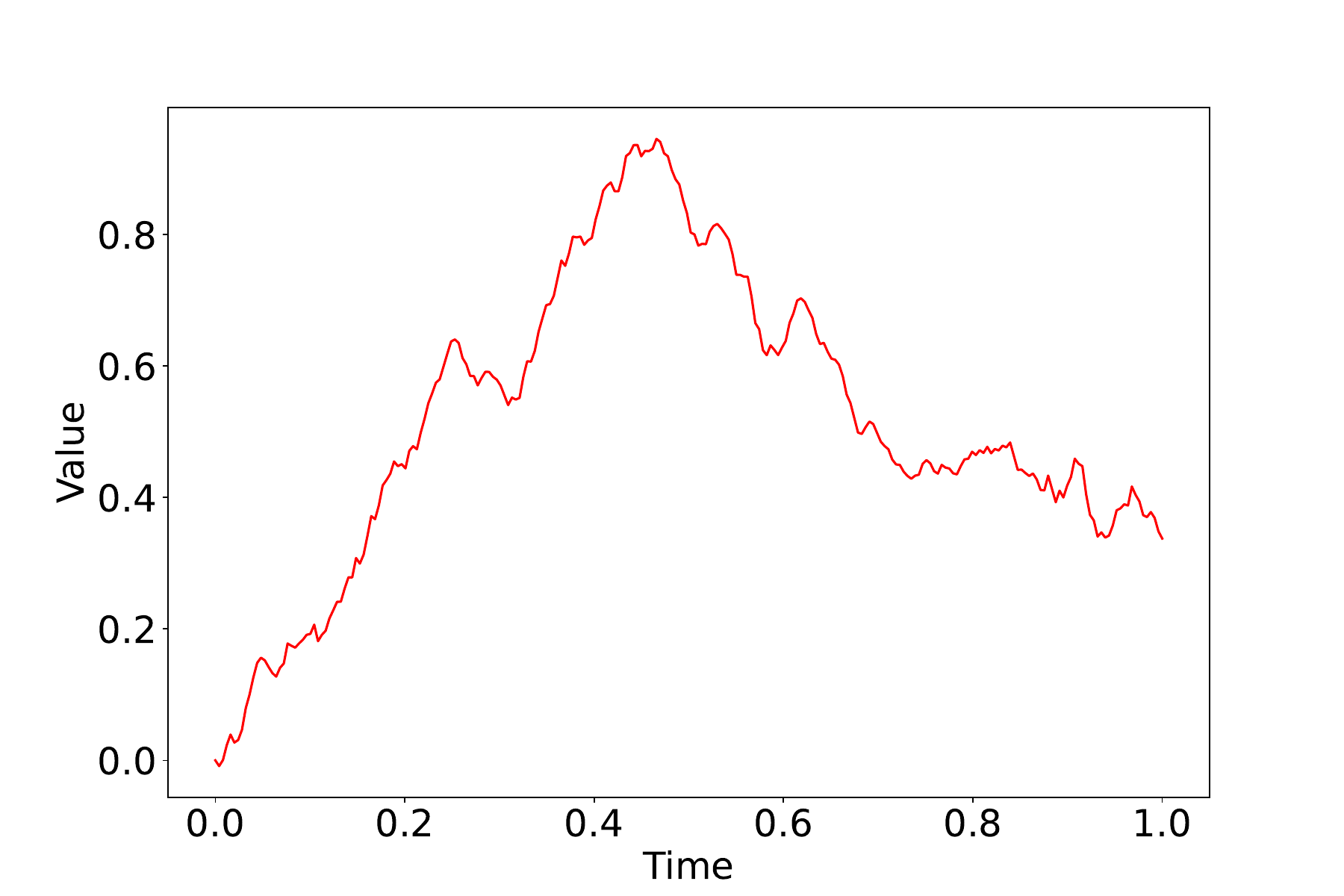}
    \caption{$H=0.75$}
  \end{minipage}
  \begin{minipage}[b]{0.19\linewidth}
    \centering
    \includegraphics[keepaspectratio, scale=0.13]{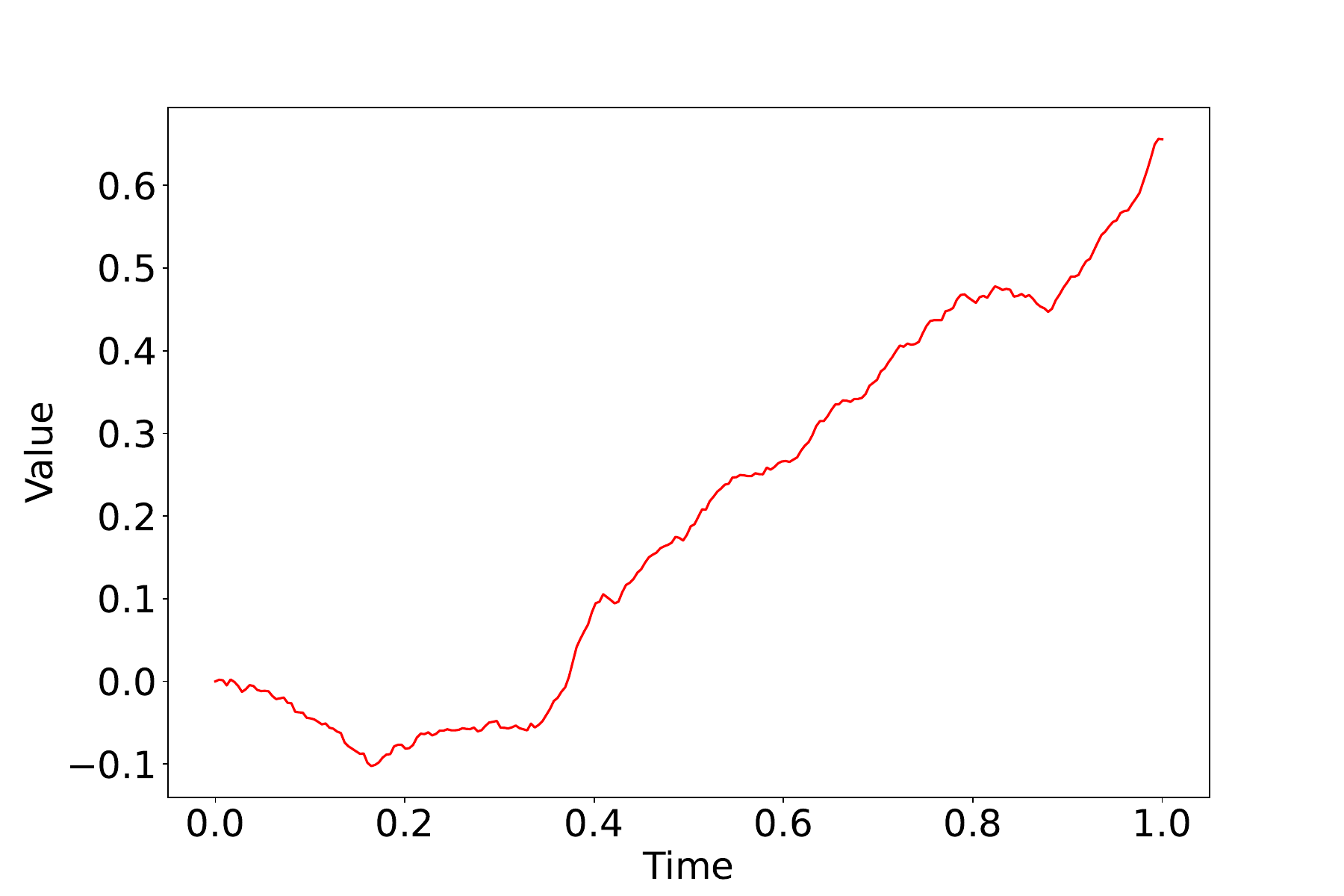}
    \caption{$H=0.9$}
  \end{minipage}
  \caption{Sample paths of fBm with various Hurst indices. The blue lines represent roughness, the green line is Brownian motion, and the red lines represent long-term dependency series.}
  \label{fig:fbm_paths}
\end{figure*}

In this section, we delve into the definition and fundamental properties of fractional Brownian motion~(fBm), a key concept in our study of the cFTM.

\begin{teigi}[fractional Brownian motion]
Fractional Brownian motion~(fBm) is a real-valued, mean-zero Gaussian process used to model the fractal characteristics of time series data. 
Let $H \in (0,1) $ be a fixed number. 
A real-valued mean-zero Gaussian process $$B^H = \{ B^H_t \}_{t \ge 0 } $$ is defined as fBm with Hurst index $H$ if it satisfies $B^H_0 = 0$ almost surely and the covariance function is given by:
 \begin{equation}
\label{eq:fbm_def}
\mathrm{Cov}(B^H_s, B^H_t) 
= \frac{1}{2} (|t|^{2H} +|s|^{2H} - |t-s|^{2H}) 
\end{equation} 
for every $s, t \ge 0 $. 
\end{teigi}
Notably, fBm reduces to standard Brownian motion when $H = \frac{1}{2}$. From fBm, we can derive fractional Gaussian noise (fGn), a discrete process defined as $\{ B^H_{t+1} - B^H_t : t = 0, 1, \ldots \}$. It follows from the definition that increments of fBm are positively correlated if $H > \frac{1}{2}$ and negatively correlated if $H < \frac{1}{2}$. Additionally, fBm possesses stationary increments, as evidenced by the distribution $$B^H_t - B^H_s \sim \mathcal{N}(0,|t-s|^{2H})$$.

Moreover, fBm adheres to the Kolmogorov-Chentsov continuity criterion\cite{bell2015kolmogorov}, implying that almost all paths of fBm with Hurst index $H$ exhibit $(H-\varepsilon)$-H\"{o}lder regularity for any $\varepsilon > 0$. This means there exists a constant $C>0$ such that:
$$
|B^H_t- B^H_s| \le C |t- s|^{H-\varepsilon} $$ for every $s,t \ge 0$. Hence, sample paths of fBm exhibit better regularity as the Hurst index increases.

The fractal properties of fBm, namely self-similarity and long-term dependency or roughness, are also crucial. Self-similarity in fBm is defined as follows:
\begin{teigi}[self-similarity]
A stochastic process $(X_t)_{t\geq0}$ is called self-similar if there exists a real number $P>0$ such that for any $a>0$ the processes $(X_{at})_{t\geq0}$ and $(a^P X_t)_{t\geq0}$ have the same finite dimensional distributions.
\end{teigi}

Lastly, we address the concepts of long-term dependency and roughness:
\begin{teigi}[Long-term dependency and Roughness]

Let $X = \{ X_t \}_{t\ge 0}$ be a (generic) process and we write its increments by $X_{s,t} = X_t- X_s$. We say that the increments of the process $X$ exhibit long-term dependency~(resp. roughness)  if for all $h>0$, 
\[
\sum_{n=1}^\infty | \mathrm{Cov} (X_{0, h}, X_{(n-1)h, nh}) |
\]
is infinite~(resp. finite).
\end{teigi}
According to the relation \eqref{eq:fbm_def}, the increments of fBm exhibit long-term dependency~(resp. roughness) if and only if $H \geq$ ~(resp. $\leq$) 1/2. 
This property can be observed in Figure \ref{fig:fbm_paths} which displays sample paths of fBm with various values of Hurst indices.

\section{Proposed Method}
\subsection{Continuous time fractional topic model}
\begin{table}[t]
\caption{Notations Used in Continuous Time Fractional Topic Model}
\centering
\begin{tabular}{cl}
\hline
\textbf{Symbol} & \textbf{Description} \\
\hline
$\mathbf{K}$ & Set of topics \\
$\mathbf{W}$ & Set of words \\
$H$ & Hurst index \\
$B^{H,(k,w)}_t$ & fBm for topic $k$ and word $w$ at time $t$ \\
$B^{H,(k)}_t$ & fBm for topic $k$ at time $t$ \\
$\alpha_{s_t}$ & Generating parameter for topics at time $t$ \\
$\beta_{s_t}$ & Generating parameter for words at time $t$ \\
$\alpha_0, \beta_0$ & Initial distributions of $\alpha$ and $\beta$ \\
$d\alpha_{s,k}, d\beta_{s,k,w}$ & SDEs for $\alpha$ and $\beta$ \\
$f_{\theta_\alpha}, f_{\theta_\beta},\sigma_\alpha, \sigma_\beta$ & parameters in SDEs \\
$s_t$ & Time stamp in the time sequence \\
$d_{s_t}$ & Document generation process at time stamp $s_t$ \\
$N_{s_t}$ & Total number of words at each time stamp $s_t$ \\
$\phi(\alpha_{s_t}), \phi(\beta_{s_t,z})$ & Categorical distributions for topics and words \\
$\sigma_V$ & Constraint set for categorical distribution parameter $\phi$ \\
$p_{\mathrm{Cat}}(x|\phi)$ & Probability density function for categorical distribution \\
$\phi(\beta)$ & Softmax function \\
$\Phi=(\Phi_\alpha ,\Phi_\beta)$ & Parameters of cFTM \\
$\Phi_\alpha, \Phi_\beta$ & Specific parameter sets for $\alpha$ and $\beta$ \\
$D_{\theta_\alpha}, D_{\theta_\beta}$ & Number of parameters in drift functions \\
$\Theta_\alpha, \Theta_\beta$ & Parameter spaces for $\alpha$ and $\beta$ \\
$L_{s_t}(\Phi)$ & Log-likelihood function at time $t$ \\
$p(w_{s}|\Phi)$ & Probability density function for a word at time $s$ \\
\hline
\end{tabular}
\end{table}

Hereafter, Let $\mathbf{K}$ be the set of topics and $\mathbf{W}$ be the set of words.
The cFTM is designed as a dynamic topic model, where the parameters of the word and topic distributions are driven by fBms and evolve over time.

Let $H\in (0,1)$ denote the Hurst index. Consider the sets of independent fBms, each starting from an initial value of $0$:
\begin{align}
  &\{ B^{H,(k,w)}_t : t\ge 0\}_{ k\in \mathbf{K}, w \in \mathbf{W}}\\
  &\{ B^{H,(k)}_t : t\ge 0\}_{k \in \mathbf{K}}
\end{align}

The generation of topic and word distributions within the cFTM is governed by a set of stochastic differential equations~(SDEs). 
These equations are critical in modeling the temporal evolution of the generating parameters $\alpha_{s_t}$ and $\beta_{s_t}$ for $t=0,\ldots,T$ in the time sequence $\{0=s_0, s_1, \ldots,s_T=T\}$. These are derived from their initial distributions $\alpha_0\in \mathbb{R}^{\mathbf{K}}$ and $\beta_0 \in \mathbb{R}^{\mathbf{K}}\times \mathbb{R}^{\mathbf{W}}$. For each topic $k\in \mathbf{K}$ and word $w\in \mathbf{W}$, the SDEs are defined as follows:

\begin{equation}
\begin{aligned}
& d\alpha_{s,k} = f_{\theta_\alpha}(\alpha_{s,k}) ds
+ \sigma_\alpha dB^{H,(k)}_s , \\ 
& d\beta_{s,k,w} = f_{\theta_\beta} (\beta_{s,k,w}) ds
+ \sigma_\beta dB^{H,(k,w)}_s .
\end{aligned}
\label{eq:alpha_beta_sde}
\end{equation}
where $\sigma_\alpha,\sigma_\beta\in\mathbb{R}$ and $f_{\theta_\alpha}, f_{\theta_\beta}$ are functions with parameters $\theta_\alpha$ and $\theta_\beta$ respectively.
Assume that the SDE of equation \eqref{eq:alpha_beta_sde} has unique solution $\{(\alpha_s,\beta_s):s \ge 0\}$ with good enough properties such that its density is well-defined~\cite{duncan2000stochastic,biagini2008stochastic}.

The design of these SDEs allows the cFTM to capture and replicate the long-term dependency and roughness characteristics inherent in the noise profile of fractional Brownian motion. 
By solving the continuous SDEs over time, the cFTM can model the dynamic topic model with a level of granularity that is both precise and flexible.

By solving the continuous SDE with respect to time, we can model the dynamic topic model with arbitrary granularity as well as \cite{wang2008continuous}.
After solving equation \eqref{eq:alpha_beta_sde}, we use its values at given time stamps $s=s_0,\ldots, s_T$. 

The generating process of documents within the cFTM framework, denoted as $$d_{s_t}=\{ (w^i_{s_t, k_i})\{1 \le i \le N_{s_t}\}:
k_i \in \mathbf{K}, w^i_{s_t,k_i} \in \mathbf{W} \}$$, where $N_{s_t}$ is the total number of words at each time stamp $s_t$, can be written as follows:

\begin{enumerate}
\item At each time stamp $s_t$, generate $\alpha_{s_{t+1},k}$ and $\beta_{s_{t+1},k,w}$ for each topic $k \in \mathbf{K}$ and word $w \in \mathbf{W}$ using the SDEs defined in equation \eqref{eq:alpha_beta_sde}.

\item For each word $w \in \mathbf{W}$ at time stamp $s_t$:
\begin{enumerate}
\item Select a topic $z \in \mathbf{K}$ from the categorical distribution~(topic distribution) $\phi(\alpha_{s_t})$.
\item Then choose a word $w$ from the categorical distribution~(word distribution) $\phi(\beta_{s_t,z})$, where the topic word generation parameter $\beta_{s_t,z} = (\beta_{s_t,z,w}){w \in \mathbf{W}} \in \mathbb{R}^{\mathbf{W}}$.
\begin{align}
& z \sim \mathrm{Categorical}(\phi(\alpha{s_t})), \nonumber \\
& w \sim \mathrm{Categorical}(\phi(\beta_{s_t,z}) ). \nonumber
\end{align}
\end{enumerate}

\end{enumerate}

Here the categorical distribution $\mathrm{Categorical}(\phi)$ with parameter $\phi \in \sigma_V$ where $\sigma_V=\{x\in [0,1]^V: x_1 + \cdots + x_V=1 \}^V$ has a probability density function given by
\begin{equation}
 p_{\mathrm{Cat}}(x|\phi)=\prod_{v=1}^V \phi_v^{x_v}
\end{equation}
on $[0,1]^V$. 
The softmax function $\phi(\beta)$ is defined as
\begin{equation}
 \phi(\beta)v= \frac{\exp(\beta_v)}{\sum{1 \le v \le V} \exp(\beta_v)} .
\end{equation}

The graphical model of our cFTM is shown in Figure \ref{fig:graphical}. 
Here, $\phi^z_s$ and $\phi^w_s$ are the topic and word distributions at time $s$, respectively. 
For simplicity, the contribution of the drift function parameters $\theta_\alpha,\theta_\beta$ is omitted. 

\begin{figure}[t]
    \centering
    \includegraphics[keepaspectratio, width=0.5\textwidth]{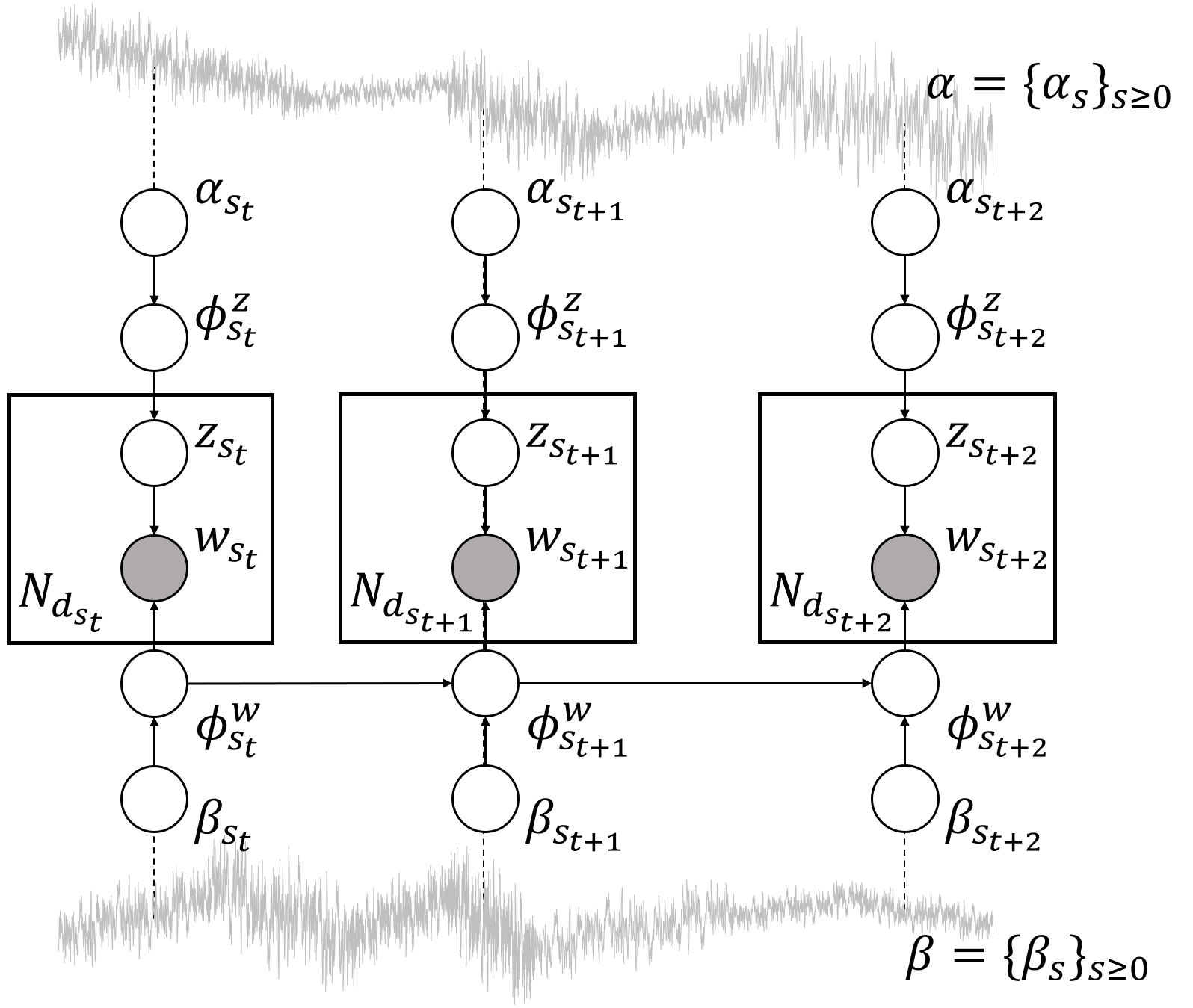}
    \caption{The graphical model of the continuous time fractional Topic Model}
    \label{fig:graphical}
\end{figure}

It is noteworthy that in the cFTM, the topic and word distributions inherently reflect the long-term dependency or roughness characteristic of the driving fBm. 
This intrinsic feature of cFTM enables it to model the complexities of topic and word evolution within documents in a way that mirrors the persistent or intermittent nature observed in real-world time-series data. 
Such a design choice enhances the model's fidelity in capturing the nuanced dynamics of topic and word distributions over time, aligning closely with the stochastic properties imparted by the fBm.
To formally state this essential characteristic of cFTM, we present the following theorem:

\begin{teiri}[long-term dependence or roughness of cFTM]\label{cor_cFTM}
Topic and word distributions in cFTM have long-term dependency or roughness.    
\end{teiri}
\begin{proof}
    The proof is given in the Appendix.
\end{proof}

\subsection{Parameter estimation}
Suppose documents $$\hat d_{s_t}=\{ \hat w^i_{s_t}\in \mathbf{W}: i=1,\ldots,|\hat{d}_{s_t}| \}~(t= 0,\ldots,T)$$ are observed at $T+1$ time stamps.
The initial distribution of the generated parameters is $\alpha_0\sim \mathcal{N}(\mu_\alpha,\nu_\alpha I_{|\mathbf{K}|})$ and $\beta_0\sim \mathcal{N}(\mu_\beta,\nu_\beta I _{|\mathbf{K}|\times |\mathbf{W}|})$ where $\mathcal{N}(\mu,v)$ denotes the normal distribution of the mean $\mu$ and variance-covariance matrix $v$.

The parameters of the cFTM we focus on are $\Phi=(\Phi_\alpha ,\Phi_\beta)$ where $\Phi_\alpha= (\mu_\alpha,\nu_\alpha,\sigma_\alpha,\theta_\alpha )$ and $\Phi_\alpha= (\mu_\beta,\nu_\beta,\sigma_\beta,\theta_\beta)$.

In this setting, let $D_{\theta_\alpha}$~(resp. $D_{\theta_\beta}$) be the number of parameters of the drift function $f_{\theta_{\theta_\alpha}}$~(resp. $f_{\theta_\beta}$).
The parameter space is $\Theta=\Theta_\alpha\times\Theta_\beta$ where 

\begin{align}
&\Theta_\alpha= 
\annot{\mathbb{R}^{\mathbf{K}}}{$\mu_\alpha$}
\times \annot{\mathbb{R}_+}{$\nu_\alpha$}
\times \annot{\mathbb{R}}{$\sigma_\alpha$} 
\times \annot{\mathbb{R}^{D_{\theta_\alpha}}}{$\theta_\alpha$} \\
&\Theta_\beta= 
\annot{\mathbb{R}^{\mathbf{K}\times\mathbf{W}}}{$\mu_\beta$}
\times \annot{\mathbb{R}_+}{$\nu_\beta$}
\times \annot{\mathbb{R}}{$\sigma_\beta$} 
\times \annot{\mathbb{R}^{D_{\theta_\beta}}}{$\theta_\beta$},
\end{align}
In the following, we maximize the log-likelihood function at each time $t=0,\ldots,T$ for the parameter $\Phi\in \Theta$ as follows.

\begin{equation}
\label{eq:loglikelihood_marginal}
\begin{aligned}
L_{s_t}(\Phi)
=\log p(\hat{d}_{s_t}|\Phi) 
= \sum_{\hat{w}_{s_t} \in \hat{d}_{s_t}}  \log p(\hat{w}_{s_t}|\Phi) 
\end{aligned}
\end{equation}

Here, for each time $s$, the probability density function for each word given the parameters $\Phi=(\Phi_\alpha,\Phi_\beta),\Phi_\alpha=(\mu_\alpha,\nu_\alpha,\sigma_\alpha),\Phi_\beta=(\mu_\beta, \nu_\beta, \sigma_\beta)$ is written as

\begin{equation}
\label{eq:word_probability}
\begin{aligned}
p(w_{s}|\Phi) 
&= \sum_{k\in\mathbf{K}} p(z_s=k|\Phi) p(w_{s,k}=w_{s_t}|\Phi)\\
&= \int_{\mathbb{R}^{\mathbf{K}}} \int_{\mathbb{R}^{\mathbf{K}\times\mathbf{W}}} \sum_{k\in\mathbf{K}} p_{\mathrm{Cat}}(z_s=k|\phi(\alpha_s))\\
&\quad\times p_{\mathrm{Cat}}(w_{s,k}=w_{s}|\phi(\beta_{s,k})) 
\\&\quad\times 
p(\alpha_s|\Phi_\alpha)
p(\beta_{s}|\Phi_\beta)
d\alpha_s d\beta_s .\\
\end{aligned}
\end{equation}
where $p(\alpha_s|\Phi_\alpha)$ and $p(\beta_s|\Phi_\beta)$ are the probability density functions of $\alpha_s$ and $\beta_s$ at time $s$.

The objective of this paper is to confirm that the proposed cFTM method reproduces the long-range dependency and roughness of a topic and word distribution.
We assume a sentence generation process under a simplified setting. The following holds for the maximization of likelihood \eqref{eq:word_probability}.

\begin{teiri}\label{est_cFTM}
    When the drift function is $f_\alpha=0$ and $f_\beta=0$, the optimization of the likelihood \eqref{eq:word_probability} is equivalent to the classical topic model~(LDA) optimization problem.
\end{teiri}
\begin{proof}
    The proof is given in the Appendix.
\end{proof}

Therefore, we apply estimation methods such as variational Bayesian estimation as well as LDA.

\begin{figure}[t]
  \begin{minipage}[b]{1\linewidth}
    \centering
    \includegraphics[keepaspectratio, scale=0.27]{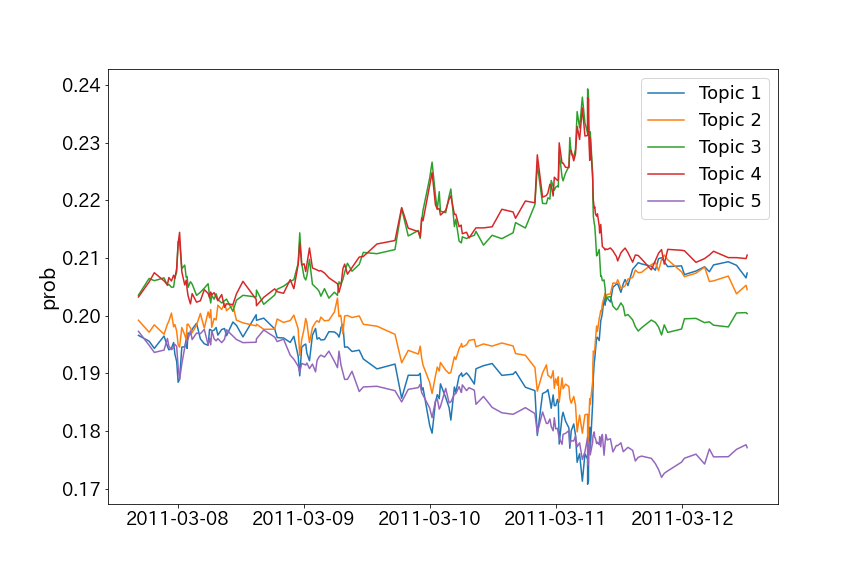}
    \subcaption{$H=0.1$ (cFTM with roughness)\newline}
  \end{minipage}
  \begin{minipage}[b]{1\linewidth}
    \centering
    \includegraphics[keepaspectratio, scale=0.27]{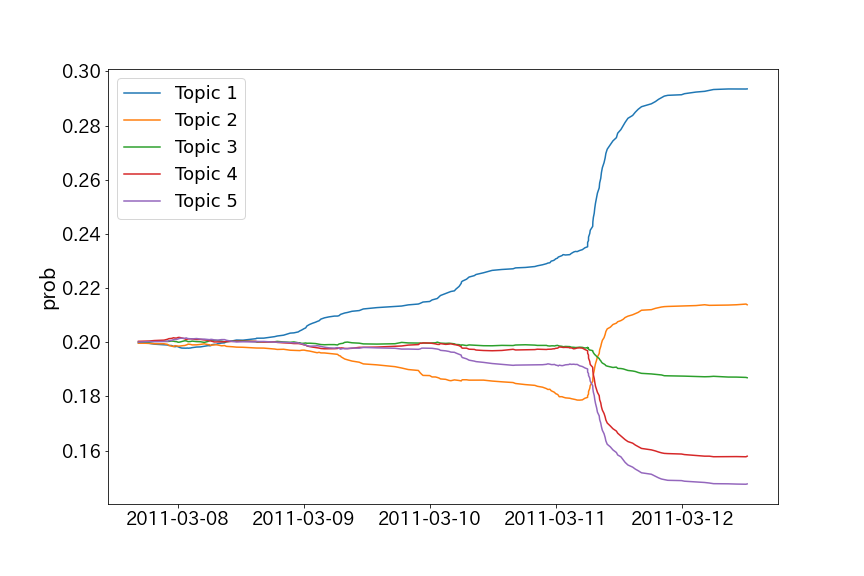}
\subcaption{$H=0.5$ (cDTM)\newline
}
  \end{minipage}
  \begin{minipage}[b]{\linewidth}
    \centering
    \includegraphics[keepaspectratio, scale=0.27]{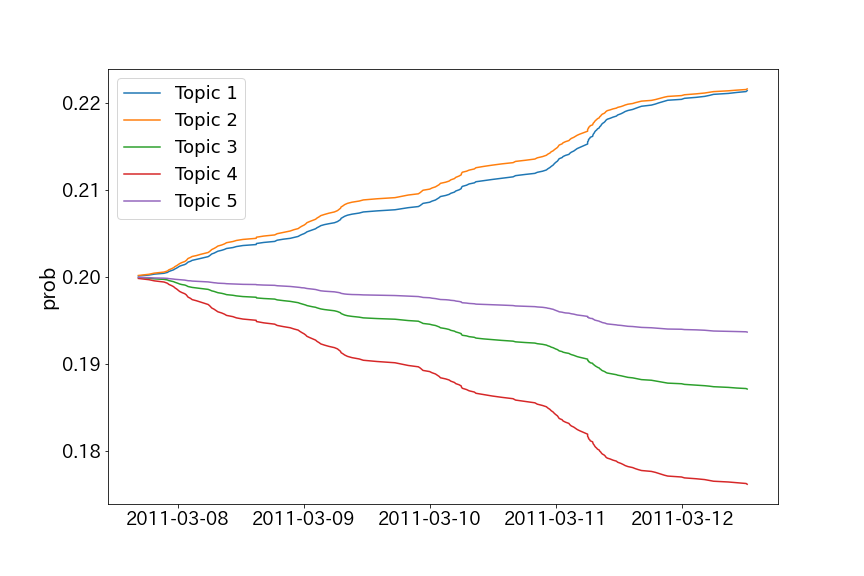}
    \subcaption{$H=0.9$ (cFTM with long-term dependency)}
  \end{minipage}
  \caption{Changes in the topic distribution before and after the Great East Japan Earthquake}
  \label{fig:topic_movement_jpn}
\end{figure}

\begin{figure}[t]
  \begin{minipage}[b]{1\linewidth}
    \centering
    \includegraphics[keepaspectratio, scale=0.27]{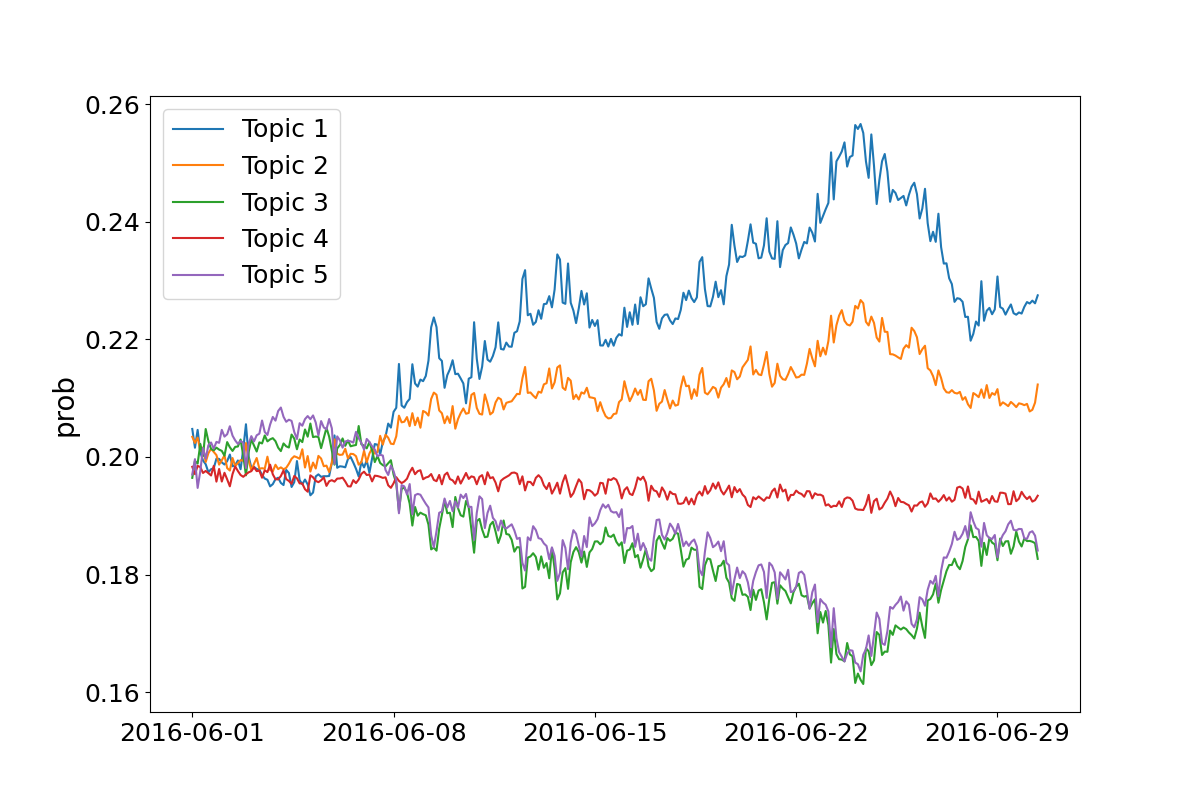}
    \subcaption{$H=0.1$ (cFTM with roughness)\newline}
  \end{minipage}
  \begin{minipage}[b]{1\linewidth}
    \centering
    \includegraphics[keepaspectratio, scale=0.27]{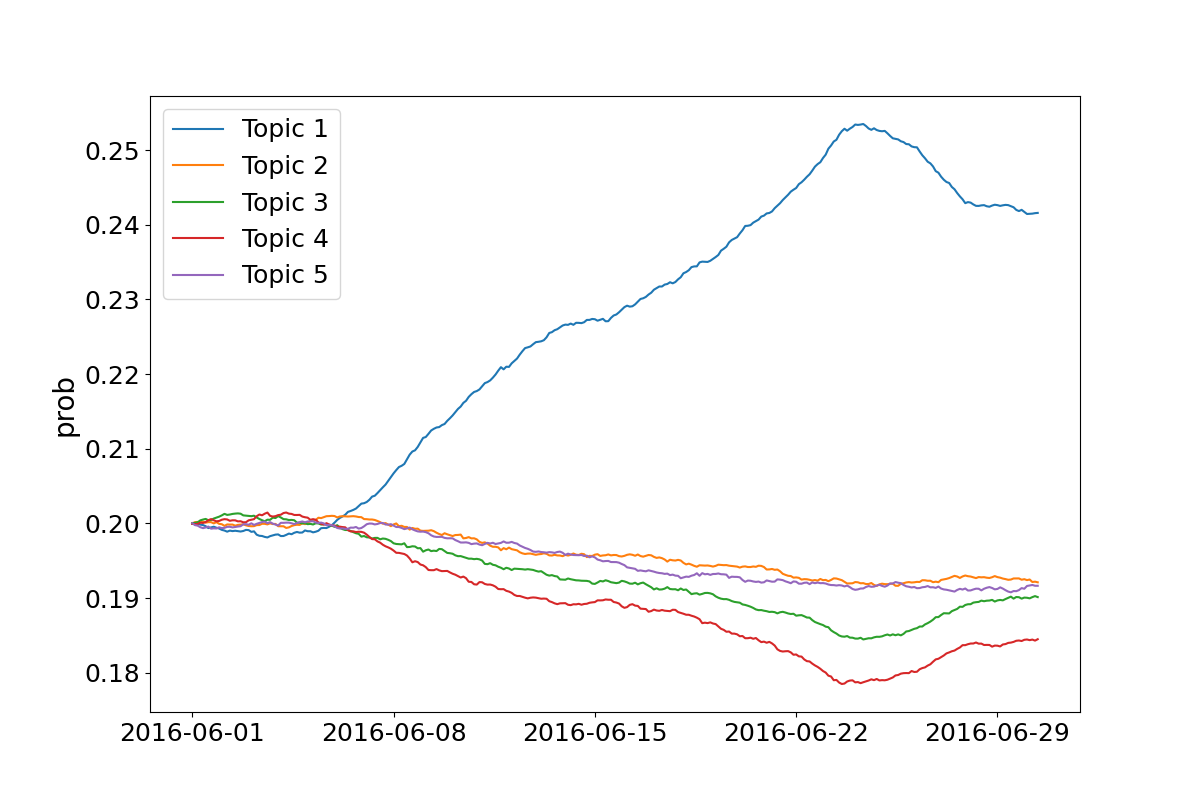}
\subcaption{$H=0.5$ (cDTM)\newline
}
  \end{minipage}
  \begin{minipage}[b]{\linewidth}
    \centering
    \includegraphics[keepaspectratio, scale=0.27]{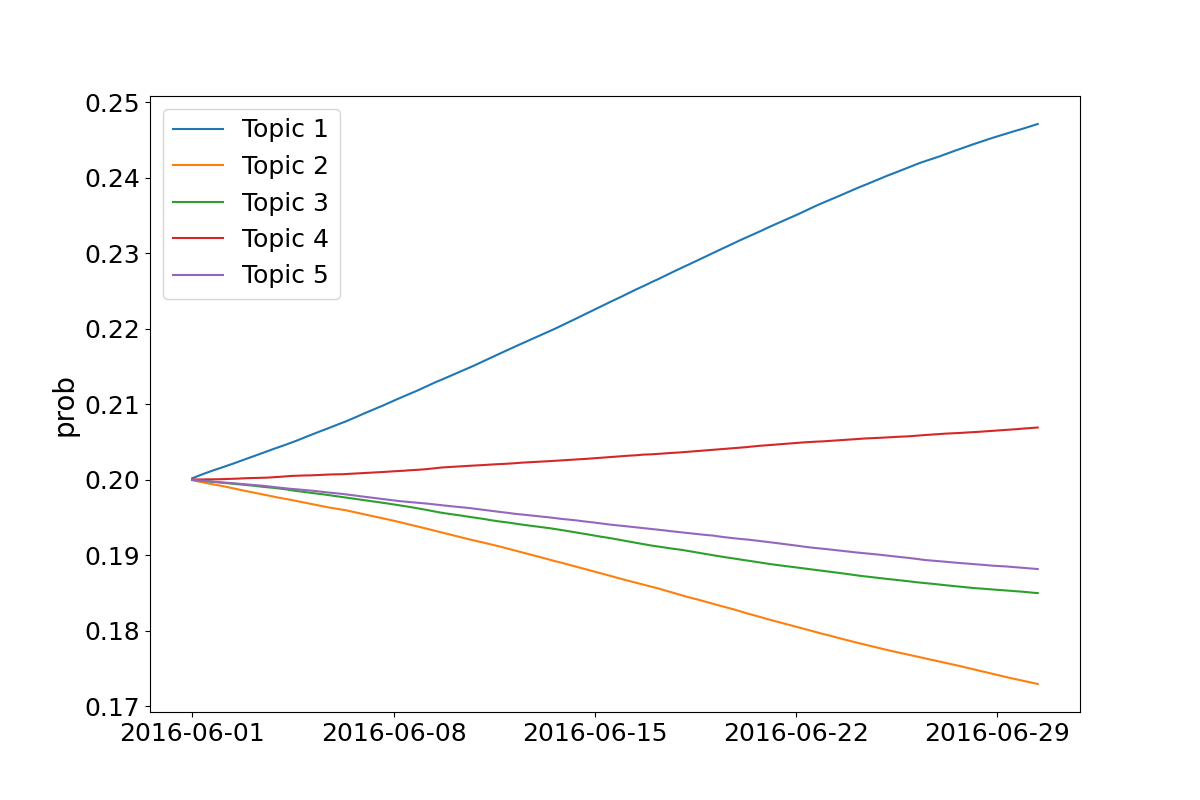}
    \subcaption{$H=0.9$ (cFTM with long-term dependency)}
  \end{minipage}
  \caption{Changes in the topic distribution before and after the Brexit}
  \label{fig:topic_movement_uk}
\end{figure}

\begin{figure}[t]
  \begin{minipage}[b]{\linewidth}
    \centering
    \includegraphics[keepaspectratio, scale=0.27]{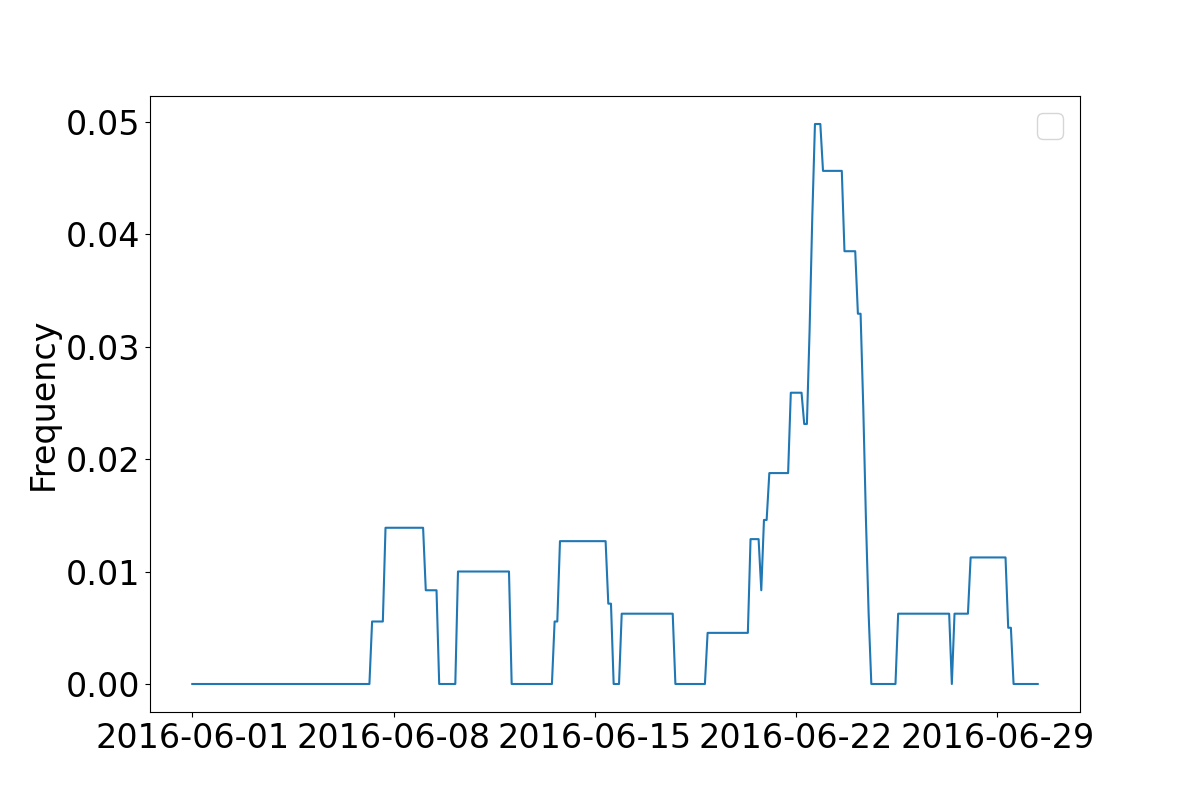}
  \end{minipage}
  \caption{Frequency of appearance of the word "referendum" in the news before and after Brexit}
  \label{fig:freq_brexit}
\end{figure}

\section{Empirical Study}
\subsection{Dataset}
To assess how well our cFTM can extract topics considering long-term dependency or roughness, we carry out a qualitative evaluation. 
This involves analyzing the changes in topic distribution before and after a major event. This method allows us to observe how topics evolve over time and how significant events impact these changes, showcasing the cFTM's ability to capture the dynamics in topic distribution.

To be specific, we apply our cFTM to two major events.
\begin{itemize}
\item Firstly, we estimate topic distributions using news articles from the five days before and after \textit{the Great East Japan Earthquake}. 
\item Secondly, we conducted a similar analysis with news from the 30 days preceding and following \textit{Brexit}. 
\end{itemize}

For our experiment, we selected specific news articles from Reuters\footnote{\url{https://jp.reuters.com/}} to analyze the two major events. 
For \textit{the Great East Japan Earthquake}, we used 178 articles from March 8-12, 2011, covering the five days immediately before and after the earthquake. 
For the \textit{Brexit} event, we chose 316 articles from the period of June 1-30, 2016, which encompasses the 30 days surrounding Brexit. 
These articles were systematically selected at equal intervals from a larger pool of 7,357 articles published during this time. 
This method ensured a comprehensive and balanced representation of news coverage for both events, allowing for a thorough evaluation of our cFTM in tracking topic evolution related to significant global incidents.

\subsection{Experimental setting}
In our cFTM, we employ a Bag of Words(BoW) format, focusing on extracting nouns from news articles as input features. 
To refine the model's accuracy, we filter out words based on frequency: words appearing in fewer than 5 news articles or in more than 50\% of the articles are excluded. This process helps eliminate overly common or rare words that may not contribute significantly to topic analysis.

For estimating the cFTM model, we use Markov chain Monte Carlo~(MCMC\cite{gamerman2006markov}) methods. Given the brief duration of our target period, only 5 days, we anticipate that the topic distribution would undergo changes during this period, while the word distribution within each topic would remain relatively stable. 
Consequently, we allow only the topic distribution generation parameter, \(\alpha\), to vary over time.

To explore different dynamics of topic evolution, we adjust the Hurst index \(H\), a crucial hyperparameter in our model, to three distinct values: 0.1, 0.5, and 0.9. 
These values represent different characteristics of topic behavior: \(H=0.1\) indicates a high level of roughness in topic evolution, \(H=0.5\) simulates Brownian motion-like behavior (akin to a previous cDTM), and \(H=0.9\) suggests strong long-term dependency in topic distribution(see also Figure \ref{fig:fbm_paths}). 
We also set the number of topics, \(K\), to 5, allowing for a detailed yet manageable examination of topic distribution changes.

\subsection{Results}
Figure \ref{fig:topic_movement_jpn} illustrates the transition of topic distribution paths modeled by our proposed cFTM around \textit{the Great East Japan Earthquake}. 
When the Hurst index \(H\) is set high, at \(H=0.9\), the model shows no significant changes in topic transitions before and after the event, indicating that each topic maintains a positive correlation or long-term dependency. 
In contrast, with a lower Hurst index setting at \(H=0.1\), a marked shift in topic distributions is observed following the event. The case with \(H=0.5\), which aligns with the cDTM scenario, falls between these two extremes. 
This demonstrates that by manipulating the Hurst index, our proposed model can effectively capture the trends of topics characterized by either long-term dependency or roughness.

For \(H=0.1\), our model effectively captures shifts in topic distribution due to an increase in related news coverage. 
Specifically, the five news topics identified by our model can be interpreted as follows: topic 1 (represented by a blue line) corresponds to Japan-related news, topic 2 (orange line) to the Japanese stock market, topic 3 (green line) to foreign news excluding Japan, topic 4 (red line) to market conditions in countries other than Japan, and topic 5 (purple line) to China-related news. 
Prior to the event, the dominant topics pertain to foreign news and market conditions, with a focus on issues such as the uncertain situation in Libya and trends in China's government and economy. 
However, after the earthquake, there is a marked increase in topics related to Japanese news and the Japanese stock market, reflecting the significant impact of the earthquake\cite{ferreira2015earthquakes}.

Figure \ref{fig:topic_movement_uk} displays the transitions of topics estimated by our cFTM surrounding the \textit{Brexit} event. 
Additionally, Figure \ref{fig:freq_brexit} presents the frequency of the term "referendum" in news articles from the same period. 
Specifically, Topics 1 and 2 at \(H=0.1\), and Topic 1 at \(H=0.5\) and \(H=0.9\) are indicative of news about the referendum to leave the European Union.

As illustrated in Figure \ref{fig:freq_brexit}, news coverage related to the Brexit referendum progressively increased leading up to the referendum date as well as \cite{krzyzanowski2019brexit}. 
Following the decision to leave the EU, there was a sharp decline in articles mentioning the referendum. Reflecting these trends, our model shows that at \(H=0.1\), the probability of Brexit-related topics decreased rapidly after the decision. 
In contrast, with higher values of \(H\), the model maintained the topic probability, indicating the perception of Brexit news as a long-term trend in articles.

These results demonstrate that our proposed method can aptly replicate the properties of fBm with varying Hurst indices in the topic distribution. 
This capability aligns with the claims presented in Theorem~\ref{est_cFTM}, underscoring the efficacy of our cFTM in modeling dynamic topic distributions that reflect the inherent characteristics of fBm.

\section{Conclusion}
The contributions of this study are as follows:
\begin{itemize}
    \item We generalize the cDTM and propose a continuous-time fractional topic model~(cFTM) that takes into account the long-term dependency and roughness of the increments of the generated topic and word parameters.
    \item We show that the evolving process of topic and word distribution of cFTM has long-term memory or roughness and the parameter estimation of the proposed method is the same as that of the topic model.
    \item Numerical experiments using actual economic news data confirm that the proposed method is able to capture the roughness of topic dynamics.
\end{itemize}

The introduction of the cFTM in this study represents a significant advancement in the field of topic modeling and natural language processing. Our research holds several broader impacts, both technically and practically, in various domains such as economics, finance, and sociology described in the Related Work section.

The limitation of our study is that the drift term was not considered because the main objective was to reproduce the long-term dependency or roughness of the topic~(word) distribution.
In addition, since fBm is generally not independently incremental, we cannot efficiently compute the posterior distribution of the topic~(word) distribution based on the Kalman Filter as well as the \cite{blei2006dynamic,wang2008continuous}.

For further study, we will consider efficient calculation methods for computing the posterior distribution of topic or word distributions.
Another direction is to replace the drift term with ODE-Net\cite{chen2018neural}, SDE-Net\cite{kong2020sde} or fractional SDE-Net\cite{hayashi2022fractional}), which are trained using a nonlinear neural net function.


\appendix

\section{Appendix:Proof of Theorems}
\subsection{Proof of Theorem \ref{cor_cFTM}}
To show Theorem \ref{cor_cFTM}, we first prove the following lemma.
\begin{hodai}(Informal)
Two processes $X=\{X_t:t\ge 0\}$ and $Z=\{Z_t :t \ge 0\}$ with $Z_t=f(X_t)$ have the same long-term dependency and roughness measured by $H$ when $f$ is a smooth function whose derivatives are bounded.
\end{hodai}
\begin{proof}
We begin with the proof for the long-term dependency. 
Suppose the process $X$ exhibits the long-term dependency: 
Now, we are in a position to show that the $C_b^\infty$-function inherits the regularity of the original paths. 
Let $X_t$ has an $\alpha$-H\"{o}lder for some $\alpha\in(0,1)$. 
A function \( f: X \to Y \) is said to be \(\alpha\)-Hölder continuous if there exists a constant \( C \geq 0 \) such that for all \( x, y \in X \), the following inequality holds:
\[
|f(x) - f(y)| \leq C \cdot |x - y|^\alpha
\]
Here, \( |x - y| \) denotes the distance between \( x \) and \( y \) in space \( X \), and \( |f(x) - f(y)| \) is the distance between \( f(x) \) and \( f(y) \) in space \( Y \). The constant \( C \) is a non-negative real number, and the exponent \( \alpha \) characterizes the degree of smoothness.

Note that by Taylor's theorem, for each $x,y\in\mathbb{R}$, there exists a constant $\delta\in(0,1)$ between $x$ and $y$ such that 
\begin{equation*}
f(x)-f(y) = f^\prime(\delta y) (x-y) . 
\end{equation*}
Then, the $\alpha$-H\"{o}lder semi-norm of $Z_t$ is estimated as 
\begin{equation*}
\begin{aligned}
\| Z \|_\alpha 
&= \sup_{s\neq t} \frac{| f(X_t) -f(X_s)|}{|t-s|^{\alpha}}\\
&\le 
\|f^\prime\|_\infty \sup_{s\neq t} \frac{| X_t -X_s|}{|t-s|^{\alpha}}.
\end{aligned}
\end{equation*}
Hence the transformed process $Z$ exhibits the same roughness as the original one. 
\end{proof}

The softmax function $\phi(\beta)$ is Lipschitz continuous~(proposition 4 of \cite{gao2017properties}) and the Lipschitz continuous functions have bounded first-order derivatives.
That is, for all $z, z^{'} \in \mathbb{R}^n$, $$||\phi(z)-\phi(z^{'})||_2 < \lambda |z-z^{'}|$$ where $\lambda$ is the Lipschitz constant.

Thus, from the above Lemma, cFTM has long-term dependence or roughness measured by the Hurst index $H$ because the two stochastic processes $X_t$ and $Z_t = f(X_t)$ have the same long-term dependency or roughness where the softmax function is a smooth function with bounded first-order derivatives.

\subsection{Proof of Theorem \ref{est_cFTM}}
Assume that the drift function is $f_\alpha=0$ and $f_\beta=0$. 
The equation~\eqref{eq:alpha_beta_sde} can be solved as follows:

\begin{equation*}
\begin{aligned}
    \alpha_s = \alpha_0 + \delta B^H_s,\quad
    \beta_s = \beta_0 + \sigma B^H_s.
\end{aligned}
\end{equation*}
Then, the distribution of $\alpha_s$ and $\beta_s$ at time $s$ can be explicitly expressed using the Gaussian distribution.
This is equivalent to the classical~(i.e., stationary with no time evolution) topic model optimization problem, and thus the optimization of the likelihood \eqref{eq:word_probability} is equivalent to the topic model optimization problem.

\end{document}